\documentclass[lettersize,journal]{IEEEtran}
\IEEEoverridecommandlockouts     
\usepackage{algorithm}
\usepackage{array}
\usepackage{enumitem}

\usepackage{stfloats}
\usepackage{tabularx}
\usepackage{url}
\usepackage{verbatim}
\usepackage{subcaption}
\usepackage{cite}
\usepackage{stackengine}
\usepackage{booktabs}
\usepackage{algorithmic}
\usepackage{graphics} 
\usepackage{epsfig} 
\usepackage{textcomp}
\usepackage{xcolor}
\usepackage{amsthm}
\usepackage{bm}
\usepackage{comment}
\usepackage[font=small,labelfont=bf]{caption}
\usepackage[symbol]{footmisc}

\usepackage{amsmath,amssymb,amsfonts}
\usepackage{textcomp}
\usepackage[colorlinks=True,allcolors=blue]{hyperref}
\usepackage{cleveref}
\usepackage[margin=1in]{geometry}
\usepackage{tcolorbox} 
\tcbuselibrary{breakable}
\makeatletter
\def\tcb@parbox@use@false{%
  \def\@parboxrestore{\linewidth\hsize\let\@parboxrestore=\tcb@parboxrestore}%
}
\makeatother

\usepackage{nicefrac}

\newtheorem{remark}{Remark}
\newtheorem{lemma}{Lemma}

\newtheorem{definition}{Definition}
\newtheorem{corollary}{Corollary}
\newtheorem{theorem}{Theorem}

\renewcommand{\ddddot}[1]{%
  {\mathop{\kern\z@#1}\limits^{\vbox to-1.4\ex@{\kern-\tw@\ex@
   \hbox{\normalfont....}\vss}}}}

\DeclareMathOperator*{\argmin}{arg\,min}

\def\BibTeX{{\rm B\kern-.05em{\sc i\kern-.025em b}\kern-.08em
    T\kern-.1667em\lower.7ex\hbox{E}\kern-.125emX}}

\makeatletter
\renewcommand{\subsubsection}{%
  \@startsection{subsubsection}{3}{0.93\parindent}%
  {0.1ex plus 0.1ex minus 0.1ex}%
  {0.1ex}%
  {\normalfont\normalsize\itshape}}
\makeatother

\begin{document}
 
\title{Boundary Sampling to Learn Predictive Safety Filters via Pontryagin's Maximum Principle}

\author{James Dallas$^{1*}$, Thomas Lew$^{1}$, John Talbot$^{1}$, Jonathan DeCastro$^{1}$, Somil Bansal$^{2}$, and John Subosits$^{1}$
\thanks{$^{1}$ Toyota Research Institute, Los Altos, CA 94022, USA. Contact emails:
\{james.dallas, john.talbot, thomas.lew, jonathan.decastro, john.subosits\}@tri.global.}

\thanks{$^{2}$ Department of Aeronautics and Astronautics, Stanford University, Stanford, CA 94305, USA. Contact email: somil@stanford.edu.}

\thanks{Corresponding author james.dallas@tri.global.}     
}

\maketitle

\begin{abstract}
Safety filters provide a practical approach for enforcing safety constraints in autonomous systems. 
While learning-based tools scale to high-dimensional systems,  
their performance depends on informative data that includes states likely to lead to constraint violation, which can be difficult to efficiently sample in complex, high-dimensional systems. 
In this work, we characterize trajectories that barely avoid safety violations using the Pontryagin Maximum Principle. 
These boundary trajectories are used to guide data collection for learned Hamilton-Jacobi Reachability, concentrating learning efforts near safety-critical states to improve efficiency. 
The learned Control Barrier Value Function is then used directly for safety filtering.  
Simulations and experimental validation on a shared-control automotive racing application 
demonstrate PMP sampling improves learning efficiency, yielding faster convergence, reduced failure rates, and improved safe set reconstruction, with wall times around $3\, \mathrm{ms}$.
\end{abstract}



\section{Introduction}

As autonomous systems become increasingly widespread, ensuring their safety is critical. Rather than embedding safety constraints within a nominal controller, one can enforce closed-loop constraint satisfaction with a downstream safety filter \cite{WABERSICH2021109597, ames2019control}. \

Control barrier functions (CBFs) are a widely used framework for safety filtering. Typically implemented as constraints in a quadratic program (QP) \cite{ames2019control}, CBF-based safety filters minimally modify the output of a nominal controller while guaranteeing that the system remains within a prescribed safe set. CBFs   have demonstrated success in many applications such as adaptive cruise control, lane keeping, and vehicle stability control \cite{ames2019control, AVEC2024,Dallas_2025}.
Despite these successes, constructing valid CBFs often requires expert design, and standard CBF-based safety filters are inherently myopic, as they enforce safety using only instantaneous information about the system state which can lead to abrupt, suboptimal, intervention and failure over a horizon \cite{GARG2024100945}. 

A response to addressing the myopic behavior of CBFs is to incorporate prediction, such as by embedding CBF constraints into model predictive control formulations  \cite{dtcbf},  enforcing the safety filtering constraints over a finite horizon \cite{WABERSICH2021109597}, or lifting the CBF conditions to trajectory space \cite{TSCBF}. Although these methods introduce lookahead predictions into the safety filter, they do so at the cost of increased computational complexity and reliance on nonlinear optimization. These methods also still require finding a valid CBF, which is difficult for complex systems. 
 
Hamilton-Jacobi (HJ) reachability provides a complementary, \textit{constructive} perspective by explicitly encoding trajectory lookahead through the computation of backward reachable tubes \cite{HJOverview}. Control barrier value functions (CBVFs) \cite{CBVF} further link HJ reachability with CBF-based safety filtering, enabling predictive safety guarantees within a QP framework. Exact HJ methods, however, suffer from the curse of dimensionality,  limiting their applicability to high-dimensional systems \cite{HJOverview}. 
Learning-based approximations of HJ Reachability \cite{deepreach} (and of other partial differential equations \cite{RAISSI, MCCLENNY2023111722}) improves scalability; however, such approaches can struggle to accurately approximate value functions, particularly for high dimensional systems with complex dynamics and constraints \cite{feng2025}. Learning safety filters for these complex systems can be improved through supervision with MPC or reinforcement learning, but these approaches remain data-inefficient as they do not explicitly target boundary-defining, safety-critical states \cite{feng2025, oh2025safetyagencyhumancenteredsafety}.

\begin{figure}[!t]
    \centering
    \includegraphics[width=0.46\linewidth]{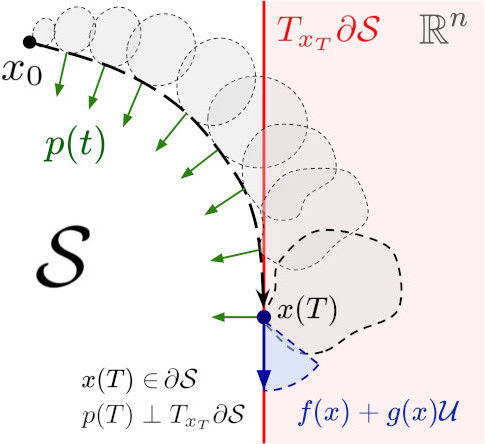}
    \includegraphics[width=0.52\linewidth]{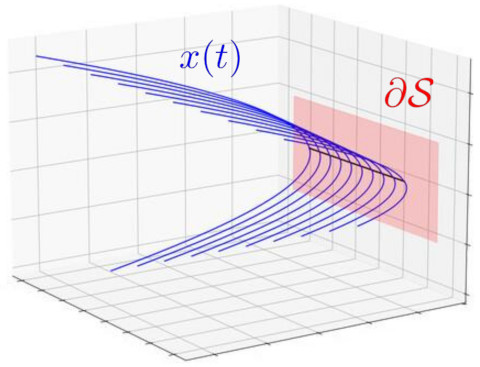}
    \caption{\textit{Left}: Extremal trajectory that barely stays in the safe set $\mathcal{S}$ and reachable sets from $x_0$. 
    \textit{Right}: Generation of boundary trajectories to learn a safety filter.}
    \vspace{-2em}
    \label{fig:extremal}
\end{figure}

As such, scalably and efficiently embedding predictive capabilities in safety filters remains an open challenge. 
In this work, we address this challenge by generating informative training samples for learned HJ Reachability. We observe that the geometry of the reachable sets is characterized by trajectories that evolve to barely satisfy safety constraints. The Pontryagin Maximum Principle (PMP) \cite{Agrachev2004,Trelat2012,BonnardChyba2003}
provides conditions that characterize such boundary trajectories, which could therefore be used to guide data collection. While the PMP has been previously leveraged to study the structure and estimation of forward reachable sets \cite{Kurzhanski2000,Natherson2025,lew2024convexhullsreachablesets}, its potential to improve the data efficiency of learned HJ reachability and CBVF-based safety filters has not been explored. 
The main contributions of this work are:
\begin{itemize}[leftmargin=6mm]
    \item We derive PMP-based optimality conditions that characterize trajectories that barely avoid constraint violation, thereby defining the boundary of the backward reachable tube.
    \item We leverage these PMP boundary trajectories to guide data collection for learned HJ reachability,  significantly improving training efficiency and accuracy in complex high-dimensional systems.
    \item We construct a learned control barrier value function (CBVF) encoding future trajectory information and demonstrate its effectiveness as a predictive, minimally invasive safety filter for human-driven vehicles operating at the limits of handling.
\end{itemize}

\section{Problem Setting and Background} \label{sec:bac}
We consider systems with control-affine dynamics
\begin{flalign}\label{eq:affine}
    &\dot{x_t} = f(x_t) + g(x_t)u_t,
\end{flalign}
with state $x_t \in \mathbb{R}^{n}$, input ${u_t\in \mathbb{R}^m}$, and 
 smooth functions $f:\mathbb{R}^n\to\mathbb{R}^n$ and $g:\mathbb{R}^n\to\mathbb{R}^{n\times m}$ ($f,$ $g$ are 
 Lipschitz and differentiable with Lipschitz Jacobians). 
 
 Safety specifications are encoded as state constraints $\mathcal{S} = \{x_t \in \mathbb{R}^n : h(x_t) \geq 0\}$,
where $h:\mathbb{R}^n \to \mathbb{R}$ is continuously differentiable. The objective is to design control policies that ensure the system remains within 
$\mathcal{S}$ for all future times: $x_0\in\mathcal{S} \implies x_t\in\mathcal{S}, ~ \forall t\geq 0$.
 
\subsection{Control Barrier Functions}
An approach to enforcing safety as a state constraint is through CBFs, which ensure safety by rendering 
$\mathcal{S}$  forward invariant.
\begin{definition}[Forward invariance]\label{def:invar} The set ${\mathcal{S} \subset \mathbb{R}^n}$ is forward invariant if  ${x_0 \in \mathcal{S} \Rightarrow x_t \in \mathcal{S},  \forall t \geq 0}$.
\end{definition}
Rendering the set $\mathcal{S}$ forward invariant amounts to enforcing nonegativity of $h(x)$. Control barrier functions provide sufficient conditions to guarantee this property.

\begin{definition}
[Control barrier function (CBF) \cite{ames2019control}]\label{def:CBF} 
The continuously differentiable function 
${h: \mathbb{R}^n \rightarrow \mathbb{R}}$ is a
CBF for the system \eqref{eq:affine} on $\mathcal{S}$ if there exists a scalar $\gamma>0$ such that
\begin{equation}\label{eq:CBF_condition}
\sup_{u_t\in \mathbb{R}^m} \big(L_f h(x_t) + L_g  h(x_t) u_t\big) \ge - \gamma \, h(x_t)
\end{equation}
for all $x_t \in \mathcal{S}$,
where $\dot{h}(x_t,u_t):=L_f h(x_t) + L_g  h(x_t) u_t$ is defined by the Lie derivatives ${L_f h}$ and ${L_g h}$.
\end{definition}
\begin{theorem}[CBFs and forward invariance \cite{ames2019control}]\label{thm:safety}
If $h$ is a CBF for system \eqref{eq:affine} on $\mathcal{S}$ for some $\gamma>0$, then any locally Lipschitz continuous controller ${\xi : \mathbb{R}^n \to \mathbb{R}^m}$ such that ${u_t=\xi(x_t)}$ satisfies
\begin{equation}\label{eq:safety_cond_control}
 L_f h(x_t) + L_g  h(x_t) u_t \geq - \gamma \,h(x_t)
\ \ \forall x \in \mathcal{S}
\end{equation}
 renders the set $\mathcal{S}$ forward invariant, i.e., safe w.r.t.~\eqref{eq:affine}.
\end{theorem}
In practice, the condition \eqref{eq:safety_cond_control} is enforced via a quadratic program (QP) that minimally modifies a nominal control input $u_d$ to ensure constraint satisfaction:
$$
\mathop{\min}_{u_t\in\mathbb{R}^m}
\|u_t-u_d\|^2 \ \text{s.t.}\  L_f h(x_t) + L_g  h(x_t) u_t \geq - \gamma \,h(x_t).
$$
This QP yields tractable safety filters and has found a wide range of applications. However, designing valid CBFs for complex systems remains a challenge, and standard CBFs are inherently myopic.

\subsection{Hamilton Jacobi Reachability} \label{sec: CBVF}
Hamilton-Jacobi (HJ) reachability gives a constructive approach to safety and allows explicitly reasoning over future trajectories by computing backward reachable tubes (BRT). Given an unsafe set $\mathcal{X_U}$ defined by a Lipschitz continuous function $l:\mathbb{R}^n\to\mathbb{R}$ as
$\mathcal{X_U} = \{x \in \mathbb{R}^n :\ell(x) \leq 0\}$, HJ reachability characterizes the set of initial states from which the system can avoid $\mathcal{X_U}$ over a finite horizon: $\{x\in\mathbb{R}^n\mid \exists u:\forall t\in[T,0],\,x_t\in\mathcal{X}_U^C\}$. This set is represented implicitly by the value function
\begin{equation} \label{eqn:val_func}
V(x_t) = \sup_{u(\cdot)} \inf_{t \in [T,0]} \ell(x_t)
\end{equation}

\noindent
Similar to CBFs, the zero superlevel set of $V(x_t)$, $\{x\in\mathbb{R}^n:V(x_t)\geq 0\}$  defines the set of states that will always remain safe over the horizon $[T,0]$.

A challenge of HJ reachability is that controllers derived from the solution do not permit a decrease in the safety metric $V$, leading to conservativeness that is typically addressed with least restrictive safety filtering in which the safety filter is only activated near the boundary \cite{kim2025reachabilitybarriernetworkslearning}. Recently, works address this challenge through Control Barrier Value Functions (CBVF) which allow a decrease in $V$ far from the boundary akin to CBFs \cite{CBVF}. Define the CBVF as:
\begin{equation} \label{eqn:cbvf}
\begin{aligned}
V_{\gamma}(x_t)  = \sup_{u_t} \inf_{t \in [T,0]} e^{\gamma(t-T)}\ell(x_t)
\end{aligned}
\end{equation}
The solution to \eqref{eqn:cbvf} is governed by the variational inequality (VI):
$
    \min \{ \ell(x_t) - V_{\gamma} (x_t), ~
    \nabla_t V_{\gamma} (x_t) + H(x_t) + \gamma V_{\gamma}(x_t) \}=0, ~
    V_{\gamma}(x_0) = \ell(x)
$
 with the Hamiltonian $H(x_t)=
\sup_{u_t} \nabla_x V_{\gamma}(x_t)\cdot (f(x_t)+g(x_t)u_t).
$

Then, a minimally invasive safety filter can  be formulated as in \cite{CBVF}:
\begin{flalign}\label{eq:costCBFQP}
&\mathbf{u}_{\rm QP}(x_t) = \argmin_{\mathbf{u_t} \in \mathbb{R}^{m} } \|u_t-u_d\|^2 \  ,
\\
 &\text{s.t.} ~\nabla_t V_{\gamma} (x_t) + L_f V_{\gamma} (x_t) + L_g V_{\gamma}(x_t) \cdot u_t  \ge -\gamma V_{\gamma}(x_t) \nonumber
\end{flalign}

\noindent Unlike CBFs, HJ reachability naturally encodes a horizon and therefore mitigates myopic behavior. However, computing 
$V$ requires solving a partial differential equation whose complexity scales exponentially with the state dimension,  limiting applicability to low-dimensional systems.

\subsection{Learned HJ Reachability}
To address the curse of dimensionality associated with HJ solvers, recent work has proposed learning-based approximations of the value function  $V$ via neural networks \cite{deepreach}. In this work we leverage DeepReach \cite{deepreach}, which learns the solution to the HJB-VI by minimizing the residual: $\mathbb{E} [\Vert \ell(x) - V_{\theta} (x, t), ~ \nabla_t V_{\theta} (x, t) + H(x, t) + \gamma V_{\theta}(x,t) \Vert]$
in a self-supervised manner. Specifically, DeepReach yields a neural network, $\Pi_{\theta}$, where the learned value function is given by $V_{\theta}(x,t) = \ell(x) + (T-t)\Pi_{\theta}(x,t)$. Curriculum training slowly propagates time backwards as training advances to optimize the network for the residual  evaluated at uniformly sampled states.

DeepReach improves computational tractability and enables scaling to higher-dimensional systems. However, it remains limited by data inefficiency: accurately approximating reachable-set boundaries grows with system complexity, and DeepReach faces particular challenge for high dimensional systems with complex dynamics and constraints \cite{feng2025}.

\section{Boundary Sampling via the\\Pontryagin Maximum Principle (PMP)} \label{sec:approach}

Boundary trajectories that barely avoid constraint violation are informative, but rare under typically-used uniform sampling, motivating a principled approach for  generating boundary-focused samples. 
Prior work \cite{feng2025} has guided learning using rollouts from model predictive control (MPC) to obtain semi-supervised labels. This approach is effective but requires solving optimization problems, increasing  computational  complexity.

In the following, we leverage the PMP to characterize such boundary-defining trajectories and exploit this structure to guide data collection for learned reachability.

\subsection{PMP Characterization of Boundary Trajectories}

\subsubsection{Minimum time hitting problem}
To characterize trajectories that define the reachable set boundary, given a fixed initial state $\bar{x}_0$ inside the safe set $\mathcal{S}$, 
we first formulate the  free-final-time optimal control problem (\textbf{OCP}):
\begin{align}
\textbf{OCP}:
\begin{cases}
\inf\limits_{(u,T)} \ 
  &T
\\[1mm]
\ \textrm{s.t.} \  
&\dot{x}_t=f(x_t)+g(x_t)u_t,
\ t\in[0,T]
\\
&x_0=\bar{x}_0,
\\
&x_T\in\partial\mathcal{S}.
\end{cases}
\nonumber
\\[-6mm]
\label{OCP}
\end{align}
Trajectories solving this problem reach the safe set boundary $\partial\mathcal{S}$ as quickly as possible. 
\textbf{OCP} is well-posed if $\partial\mathcal{S}$ is reachable from $\bar{x}_0$.
Next, we characterize optimal solutions to \textbf{OCP} using the PMP.

\subsubsection{Necessary optimality conditions from the PMP}
The PMP provides necessary optimality conditions  for \textbf{OCP}.
As the Hamiltonian of \textbf{OCP} is
$$H
(x, u, p) =
p^\top(f(x)+g(x)u),
$$
for any optimal solution $(x, u)$ to \textbf{OCP}, there exists an absolutely-continuous function $p:[0,T]\to\mathbb{R}$, called the adjoint vector, and a real number $p^0\in\{-1,0\}$, such that the nontrivality condition $(p,p^0)\neq 0$ holds, and the following optimality conditions (\textbf{OC}) hold:
\begin{align} 
\label{eq:OC_optimality_conditions}
\textbf{OC}:  
\begin{cases}
\dot{x}_t=f(x_t)+g(x_t)u_t,
\\
\dot{p}_t=-(\nabla f(x_t)+\nabla g(x_t)u_t)^\top p_t,
\\[0mm]
u_t=\mathop{\arg\max}\limits_{v\in\mathcal{U}}p_t^\top g(x_t) v,
\\
x_T\in\partial\mathcal{S}
\\
p_T\perp T_{x_T}\partial\mathcal{S},
\\
p_T	^\top(f(x_T)+g(x_T)u_T)=-p^0.
\end{cases}
\end{align}
Such a tuple $(x,p,p^0,u)$ is called an \textit{extremal}:
\begin{itemize}
\item 
If $p^0=-1$,  the extremal is called \textit{normal}. 
\item If $p^0=0$,  the extremal is called \textit{abnormal}. 
\end{itemize}In most applications, the extremal is normal, which (informally) indicates that the cost plays a role in the problem. However, we show next that the extremals of interest are actually  those that are abnormal.
\subsubsection{Abnormal extremals are those that barely stay in $\mathcal{C}$}
The following result connects abnormal extremals of the minimum-time problem \textbf{OCP} and trajectories that barely avoid constraint violation.

\begin{lemma}[Barely-staying extremals are abnormal]
\label{lem:barely_staying}
Let $(x,p,p^0,u)$  be an extremal for \textbf{OCP}, i.e., $(x,p,p^0,u)$ satisfies the optimality conditions \textbf{OC}. 
Assume that 
\begin{itemize}[leftmargin=3mm]
\item \textbf{A1} (Safe set dimension) $T_{x_T}\partial\mathcal{S}$ is of dimension $n-1$.
\item  \textbf{A2} (Solution is safe) $x_t\in\mathcal{S}$ for all $t\in[0,T]$.
\item  \textbf{A3} (Solution can barely remain in $\mathcal{S}$)
For some $\epsilon>0$, 
the control $u$ can be smoothly extended   over $(T,T+\epsilon)$ 
such that  $x_t\in\mathcal{S}$ for all $t\in[0,T+\epsilon)$. 

\end{itemize} 
Then,
$p^0=0$.
\end{lemma}
\noindent 
This result states that trajectories that can barely stay inside $\mathcal{S}$ are projections of \textit{abnormal} extremals. 
Assumption \textbf{A1} states that the tangent space of the safe set boundary is of dimension $n-1$ (e.g., only one of the state variables is constrained to be within two bounds: $\mathcal{S}=\{(x_1,x_2)\in\mathbb{R}^2:\underline{x}\leq x_1\leq \overline{x}\}$). Assumptions \textbf{A2} and \textbf{A3} state that the state trajectory is in the safe set $\mathcal{S}$ over $[0,T]$ and can later remain in the safe set $\mathcal{S}$  after touching the safe set boundary $\partial\mathcal{S}$ at time $T$ for some smooth extension of the control $u$. 
This result follows from observing that  barely-safe trajectories must hit $\partial\mathcal{S}$ with a velocity that is tangent to $\partial\mathcal{S}$ at time $T$ (Figure \ref{fig:extremal}), otherwise it would always exit $\mathcal{S}$.

\begin{remark}
Assuming that the control extension is smooth for $t\geq T$ is necessary. For example, the system $\dot{x}_t=u$ has sufficient control authority to enable hitting $\partial\mathcal{S}$ with non-tangent velocity and remain in $\mathcal{S}$ for $t>T$: Dropping the smoothness requirement in \textbf{A3} is insufficient to conclude $p^0=0$. See also Section \ref{sec:conc}.

\end{remark}
\begin{proof}
We prove Lemma \ref{lem:barely_staying} by contradiction.
Assume that the extremal is normal, so $p^0=-1\neq 0$. Then, $
p_T^\top \dot{x}_T\neq 0$. 
In particular, $p_T\neq0$ and $\dot{x}_T\neq 0$.

\noindent
\begin{minipage}[t]{0.66\linewidth}
\hspace{6pt} From \textbf{OC},  $p_T$ is perpendicular to $T_{x_T}\partial\mathcal{S}$, so  that $\dot{x}_T\notin T_{x_T}\partial\mathcal{S}$ (otherwise, we would have $p_T^\top\dot{x}_T=0$). Then, $\dot{x}_T$ points outside $\mathcal{S}$ since $x_t\in\mathcal{S}$ for  $t\in[0,T]$ (\textbf{A2}). 
By \textbf{A1}, $T_{x_T}\partial\mathcal{S}$ is of dimension $n-1$, 
so there is a Euclidean ball $B$ so that $\text{Int}(B)\cap\mathcal{S}=\emptyset$, $x_T\in\partial B$, and  $T_{x_T}\partial B=T_{x_T}\partial\mathcal{S}$.
\end{minipage}
\hfill
\begin{minipage}[t]{0.32\linewidth}
    \vspace{0pt} 
    \centering
    \includegraphics[width=0.8\linewidth]{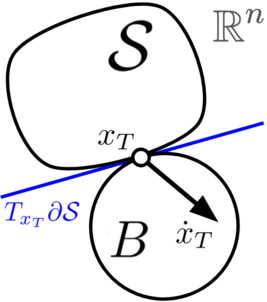}
    \captionof{figure}{Tangent ball.}
\end{minipage}

Then, by \textbf{A3}, and using  \cite[Lemma 4.6]{LewEstimatingCH2024}, no matter the choice of smooth control extension $\tilde{u}:[0,T+\epsilon]\to\mathcal{U}$ with $\tilde{u}=u$ on $[0,T]$, we have $x_{T+\epsilon}\in\text{Int}(B)$ so $x_{T+\epsilon}\notin\mathcal{S}$, so the system will always exit $\mathcal{C}$. 
This is a contradiction. 
\end{proof}  
Lemma \ref{lem:barely_staying} leads to the following result which gives conditions that barely-safe trajectories must satisfy.

\begin{corollary}
\label{cor:barely}(Characterization of barely-staying trajectories).
Let $(x,p,p^0,u)$  be an extremal for \textbf{OCP}, satisfying   \textbf{OC}. 
Assume that \textbf{A1-A3} hold and that 
\begin{itemize}[leftmargin=3mm]
\item  \textbf{A4} ($\mathcal{S}$ is a superlevel set) $\mathcal{S}=\{x\in \mathbb{R}^n:h(x)\geq 0\}$ where $h:\mathbb{R}^n\to\mathbb{R}$ is a continuously differentiable submersion at $x_T$ (i.e., $\nabla h(x_T)$ has full rank one).
\end{itemize} 
Then, $p^0=0$, so  $(x,p,p^0,u)$ satisfies
\begin{align} 
\label{eq:PMPODE}
\textbf{OC}:  
\begin{cases}
\dot{x}_t=f(x_t)+g(x_t)u_t,
\\
\dot{p}_t=-(\nabla f(x_t)+\nabla g(x_t)u_t)^\top p_t,
\\[0mm]
u_t=\mathop{\arg\max}\limits_{v\in\mathcal{U}}p_t^\top g(x_t) v,
\\
h(x_T)=0,
\\
p_T=\nabla h(x_T),
\\
p_T	^\top(f(x_T)+g(x_T)u_T)=0.
\end{cases}
\end{align}
\end{corollary}
\noindent The result follows from Lemma \ref{lem:barely_staying} and classical results  \cite[Chapter 5]{Lee2012}, since the tangent space to the boundary $\partial \mathcal{S}$ at a point $x_T \in \partial \mathcal{S}$ is given by
$
T_{x_T}\partial\mathcal{S}=\{y\in\mathbb{R}^n:\nabla h(x_T)^\top y=0\} 
$ under \textbf{A4}. 

While the PMP has been used previously to study the structure of 
forward reachable sets \cite{Kurzhanski2000,Natherson2025,lew2024convexhullsreachablesets} and time-optimal trajectories \cite{Wang2026,Boscain2005}, Corollary \ref{cor:barely} enables its use as a principled mechanism for generating boundary-focused training data.

\subsection{PMP-Informed Boundary Sampling Algorithm}Corollary~\ref{cor:barely} describes trajectories that barely remain safe, which therefore provide informative samples to learn minimally-invasive predictive safety filters. We now leverage this result to generate boundary samples.

Assume that the safe set $\mathcal{S}$ is defined as $\mathcal{S} = \{x \in \mathbb{R}^n : h(x) \ge 0\}$, where $h$ is a submersion (\textbf{A4}).  
Further, assume that the admissible control set $\mathcal{U}$ is compact and strictly convex, and that $g(x)$ is invertible. Under these conditions, the PMP optimality conditions  admit a unique closed-form solution for the control input \cite{lew2024convexhullsreachablesets}
\begin{equation}
    \label{eq:control:closed_form}
\mathop{\arg\max}\limits_{u\in\mathcal{U}} \, v^\top u
= \varphi^{\mathcal{U}}(v),
\end{equation}
where $v=g(x)p$ denotes the costate-dependent direction. 
Based on these observations, boundary-defining trajectories can be generated using the following procedure:
\begin{enumerate}[leftmargin=5mm]
    \item \textbf{Initialization}: Randomly sample initial states from the safe set $\mathcal{S}$.
    \item \textbf{Boundary point identification}: 
    For each initialization, compute terminal points $(x_T, p_T)$ on the safe set boundary via gradient descent that satisfy the PMP terminal conditions
    \begin{equation}
    \label{eq:pmp_terminal}
    \begin{aligned}
        h(x_T) &= 0, \\
        p_T -\nabla h(x_T)&= 0, \\
        p_T^\top \big(f(x_T) + g(x_T) u_T \big) &= 0,
    \end{aligned}
    \end{equation}
    where the control input $u_T$ is obtained from the closed-form PMP maximizer in \eqref{eq:control:closed_form}.
    \item \textbf{Backward integration}: From each $(x_T,p_T)$, integrate the PMP state-costate dynamics \eqref{eq:PMPODE} \textit{backward in time}. 
\end{enumerate}
The resulting state trajectories evolve along the boundary of the backward reachable tube and provide dense coverage of safety-critical regions of the state space. Samples collected along these trajectories can then augment the training dataset of DeepReach, concentrating learning capacity near boundary-defining states and improving data efficiency relative to uniform sampling.

In practice, we use the library \texttt{ChReach}  \cite{lew2024convexhullsreachablesets} to generate these trajectories, as it implements the closed-form maximizer $\varphi^{\mathcal{U}}$ in \eqref{eq:control:closed_form}  for many sets $\mathcal{U}$ and under relaxed assumptions (e.g., non-invertible $g$ and box  $\mathcal{U}$).

\section{Application} \label{sec:app}

To demonstrate the proposed approach, we consider the problem of shielding a driver (or controller) operating a vehicle at the limits of handling in a racing scenario. This setting exhibits highly nonlinear dynamics and requires proactive behavior to ensure that the vehicle enters turns at an appropriate speed while remaining within track boundaries.

The track is represented as an oval course, depicted by the dashed lines in Fig.~\ref{fig: learned_set}, consisting of $20~\mathrm{m}$ straight segments and turns with a radius of $12~\mathrm{m}$. Safety is defined as remaining within the road edges, with the unsafe set of \eqref{eqn:val_func} defined by the function $\ell(x): = 3 - \vert e \vert$ where $e$ denotes the lateral deviation from the track centerline. 

\vspace{-1em}
\subsection{Vehicle Dynamics}
To generate training data and construct the safety filter, we use a 3 degree-of-freedom single track model which describes vehicle behavior even in extreme situations such as racing\cite{dallas2023hierarchical}.
We represent the state as ${x=[s\ e\ \Delta\phi\  V\ r\ \beta\ \delta\ \tau\,]^\top}$ where $x$ contains the path progress, lateral error, course error, velocity, yaw rate, sideslip, roadwheel angle, and total torque respectively. To promote smoothness, the input is represented as ${u=[\, \dot\delta\ \dot\tau\,]^\top}$ where $u$ contains the roadwheel angle rate and total torque rate. State constraints on $[\delta , ~\tau ]$ are handled as in Appendix \ref{appendix}. The system's dynamics are given as:

{\small%
\begin{flalign}\label{eq:bike}
    f(x) &= 
    {\renewcommand{\arraystretch}{1.5}%
    \begin{bmatrix}
    \frac{V\cos(\Delta\phi)}{1-\kappa e}%
    \\%
    V\sin(\Delta\phi)%
    \\%
    \dot{\beta} + r -\frac{\kappa_{\rm ref} V\cos(\Delta\phi)}{1-\kappa_{\rm ref}e}%
    \\%
     \frac{F_{\rm xf}\cos(\delta - \beta) - F_{\rm yf}\sin(\delta - \beta) + F_{\rm xr}\cos\beta + F_{\rm yr}\sin\beta}{m}%
     \\%
    \frac{a ( F_{\rm xf}\sin\delta + F_{\rm yf}\cos\delta ) - b F_{\rm yr}}{I_z}%
    \\%
    \frac{F_{\rm xf}\sin(\delta - \beta) + F_{\rm yf}\cos(\delta - \beta) -F_{\rm xr}\sin\beta + F_{\rm yr}\cos\beta}{mV}-r%
    \\%
    0%
    \\%
    0%
    \end{bmatrix},\nonumber 
    }
\\
    g(x) &= 
    \begin{bmatrix}
    0 & 0 &0 & 0 & 0 & 1 & 0
    \\
    0 & 0 &0 & 0 & 0 & 0 & 1 
    \\
    \end{bmatrix}^{\top},
\end{flalign}
}

\noindent where $\kappa_{\rm ref}$ is the reference curvature of the track, $a$ and $b$ are the distances between the center of mass and front and rear axles, respectively, $m$ is the vehicle mass, and $I_z$ is the vehicle's moment of inertia. Tire forces are obtained from a derated brush tire model \cite{Pacejka_2005} which maps the tire slip angle to lateral forces while accounting for tire saturation.  The maximum lateral force is obtained by derating the frictional limits by the longitudinal force to account for force coupling, $F_{y, {\rm max}} = \sqrt{(\mu F_z)^2-\zeta F_x^2}$. Where $\mu F_z$ is the maximal frictional force, and ${\zeta = 0.99}$ promotes numerical stability. Additionally,  the state is extended with a virtual state with zero dynamics, representing the class-$\mathcal{K}$ parameter such that $\gamma$ can be tuned for experiments to yield the desired behavior \cite{kim2025reachabilitybarriernetworkslearning}.

\begin{figure}[t]
    \vspace{0.6em}
    \includegraphics[width=0.6\columnwidth]{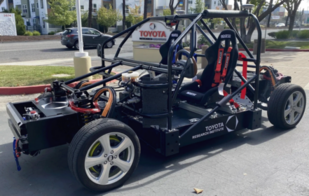}
    \centering
    \caption{Experimental platform.}
    \vspace{-1.5em}
    \label{fig:GRIP}
\end{figure}

\vspace{-1em}
\subsection{Experimental Vehicle} \label{sec: exp_veh}
The experimental vehicle (Fig.~\ref{fig:GRIP}) is a custom drive-by-wire vehicle with a steering range of ${0.71\,\mathrm{rad}}$ and Elaphe IWM M700 VD4 motors at each wheel. A  dual antenna Oxford Technical Systems (OxTS) RT3003 RTK-GPS/IMU system provides meaurements at 200 Hz and on-board computations are performed on a  RAVE ATC8110-F with an Intel Xeon E-2278GE processor. 
All experiments are conducted on a large skidpad and were performed on a closed course. 

\vspace{-0.5em}
\section{Results} \label{sec:res}
We now demonstrate that PMP-based sampling improves safety filtering performance, sample efficiency, and convergence in a high-dimensional system. Specifically,
Sec.~\ref{sec:fr} and~\ref{sec:recon} show how leveraging PMP improves learning quality, while Section~\ref{sec: exp} demonstrates proactive and safe control in a shared control racing application.

\vspace{-0.9em}
\subsection{Experimental and Proxy Network Architecture} \label{sec: Net_Arch}
\textbf{Proxy Ground-Truth Architecture} — Since the exact characterization of the
safe set is intractable for this system, we leverage the DeepReach framework to
learn a high-fidelity proxy model to serve as ground-truth. This proxy model has higher capacity, larger training
sets, and more training epochs than the experimental models detailed below.
Specifically, the proxy network consists of three hidden layers of 512 neurons and is trained for $300k$ total epochs ($20k$ pretraining, $260k$
curriculum, and $20k$ finetuning) using $110k$ samples and a cosine annealing
learning rate scheduler.

\textbf{Experimental Architecture} — To evaluate the effects of the
proposed PMP sampling strategy, we leverage the DeepReach framework to train
several experimental networks, varying capacity, sample size, horizon, and the number of training
epochs to estimate the safe set. Each network is fully-connected, consisting of
two or three hidden layers with 512 neurons each. $40\%$ of samples are drawn from the safe set interior and $60\%$ are generated with the PMP boundary sampling scheme or uniform boundary sampling as a baseline. Boundary samples are perturbed by adding uniformly sampled noise of $\pm 10 cm$ to the boundary lateral position, producing both safe and unsafe states in the vicinity of the boundary. For all comparison configurations
(e.g., sampling strategy and epochs), we train four models with
differing seeds and assess performance using a one-sided T-test. 

\begin{figure*}[t]
\centering

\begin{subfigure}[b]{0.32\textwidth}
    \centering
    \includegraphics[width=\textwidth]{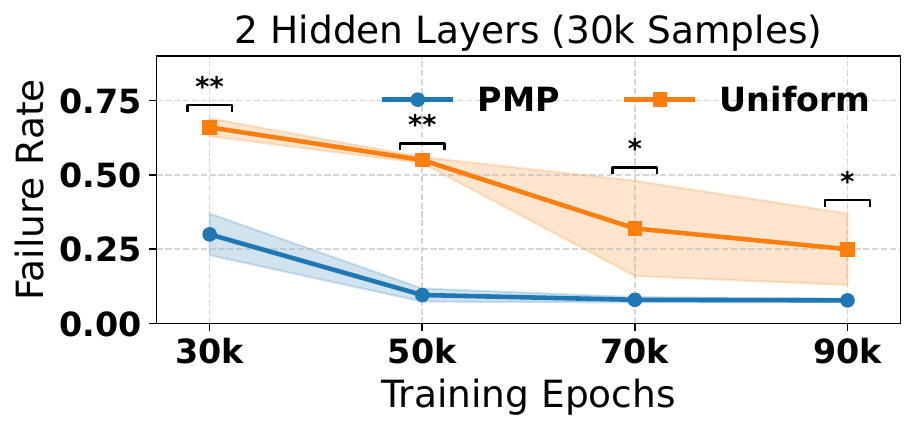}
    \caption{2L, 30k samples}
    \label{fig:2layer_epochs_sub}
\end{subfigure}
\hfill
\begin{subfigure}[b]{0.32\textwidth}
    \centering
    \includegraphics[width=\textwidth]{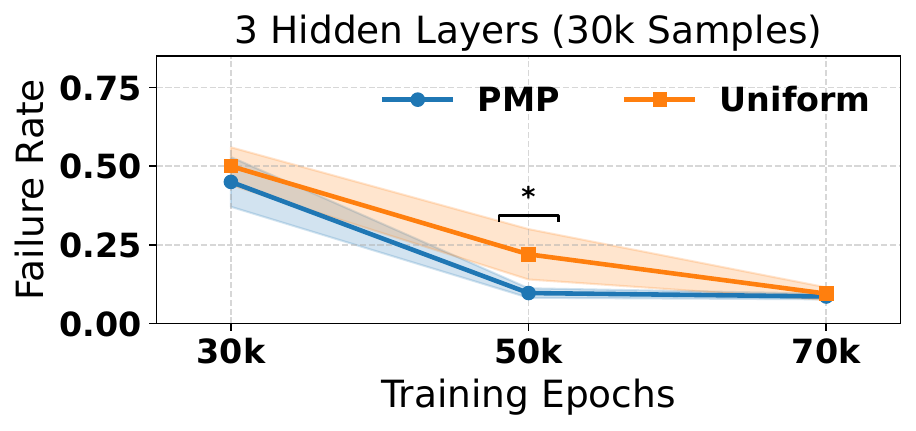}
    \caption{3L, 30k samples}
    \label{fig:3layer_epochs_sub}
\end{subfigure}
\hfill
\begin{subfigure}[b]{0.32\textwidth}
    \centering
    \includegraphics[width=\textwidth]{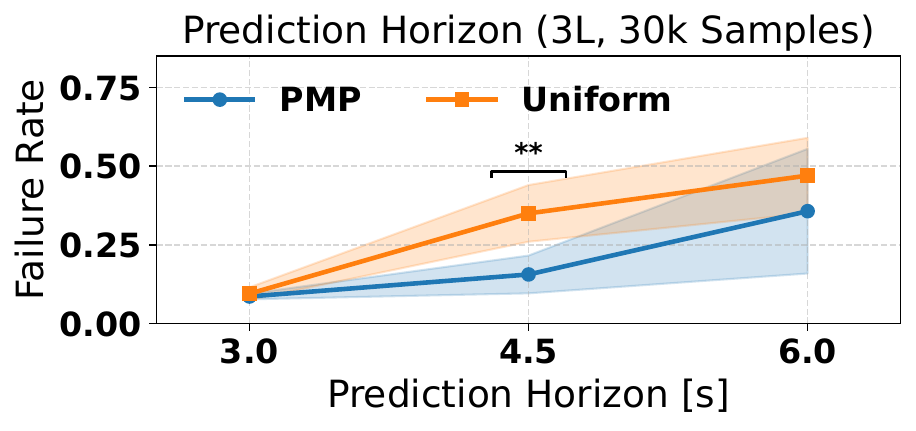}
    \caption{Prediction horizon}
    \label{fig:horizon_sub}
\end{subfigure}

\vspace{0.5em}

\begin{subfigure}[b]{0.48\textwidth}
    \centering
    \includegraphics[width=0.9\textwidth]{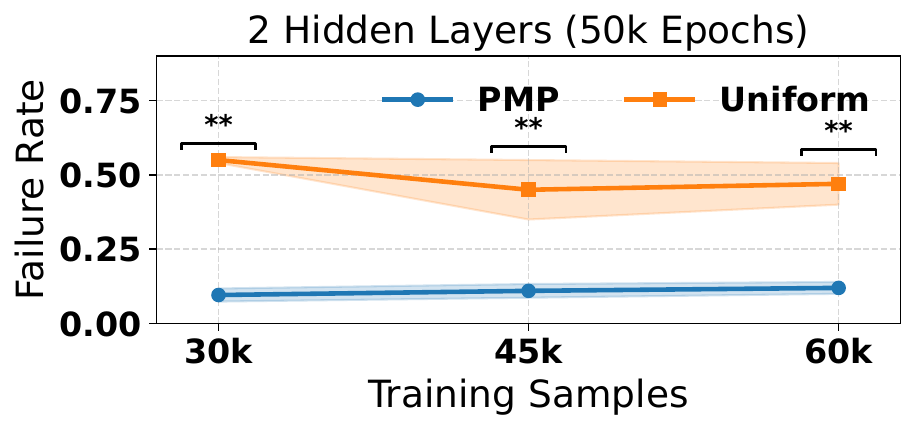}
    \caption{2L dataset size (50k epochs)}
    \label{fig:2layer_dataset_sub}
\end{subfigure}
\hfill
\begin{subfigure}[b]{0.48\textwidth}
    \centering
    \includegraphics[width=0.9\textwidth]{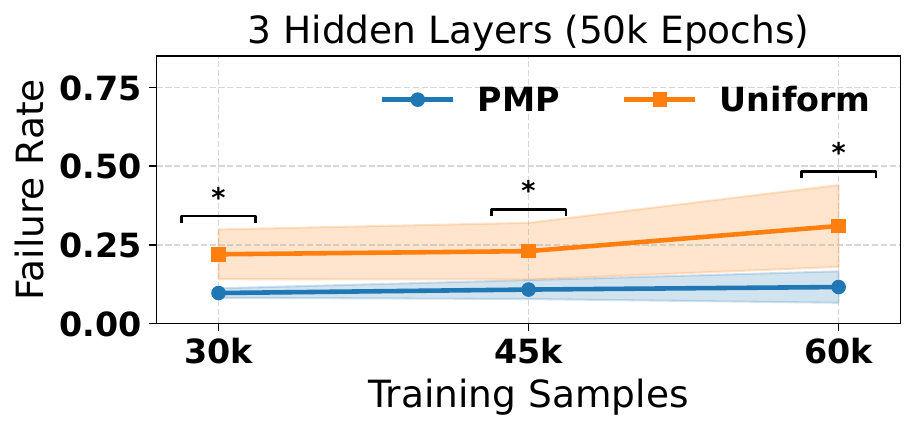}
    \caption{3L dataset size (50k epochs)}
    \label{fig:3layer_dataset_sub}
\end{subfigure}

\caption{Failure rate between PMP and uniform sampling. Performance vs training epochs and prediction horizon (top). Performance vs dataset size (bottom). PMP consistently improves failure rate. Shaded regions denote $\pm 1$ S.D. Significance shown with $^{*}p<0.05$ and $^{**}p<0.01$.}
\vspace{-1.5em}
\label{fig:main_results}
\end{figure*}

\vspace{-1em}
\subsection{PMP Sampling Improves Failure Rate} \label{sec:fr}
\noindent\textit{\textbf{Training Epochs}} --
Fig.~\ref{fig:2layer_epochs_sub} and \ref{fig:3layer_epochs_sub} depict failure rates versus various final training epochs for two- and three-layer networks. Each final epoch represents the model trained to completion. PMP sampling reduces failure rates, particularly for limited compute. At a terminal epoch of 30k, failure rate statistically significantly decreases from 0.66 to 0.3 (2-layer) and from 0.50 to 0.45 (3-layer). As we train for longer the gap narrows (achieving 0.08 failure rate). These results indicate that PMP sampling improves learning efficiency, enabling faster convergence to safer solutions.

\noindent\textbf{\textit{Dataset Size}} --
Fig.~\ref{fig:2layer_dataset_sub} and \ref{fig:3layer_dataset_sub} shows failure rate versus dataset size for a fixed training budget of 50k epochs. Across all sample and model sizes,  PMP statistically significantly outperforms uniform sampling. 
These results show with statistical significance that PMP improves sample efficiency by concentrating training on informative safety-critical boundary trajectories.

\noindent\textbf{\textit{Prediction Horizon}} --
Fig.~\ref{fig:horizon_sub} shows failure rate as the BRT horizon increases for the three-layer model with a budget of 70k epochs. As the horizon increases, PMP sampling leads to lower failure rates (e.g., 0.35$\rightarrow$0.16 at 4.5 s and 0.47$\rightarrow$0.36 at 6.0 s). 
These results indicate that PMP sampling is beneficial in longer horizon prediction tasks, where capturing the safe set boundary becomes more challenging.

\vspace{-0.9em}
\subsection{PMP Sampling Improves Safe Set Reconstruction} \label{sec:recon}
Fig. \ref{fig: learned_set} depicts the learned value function for the proxy (left), and comparison two-layer models trained with $50k$ epochs and $30k$ samples for both PMP (center), and uniform sampling (right).  
Notably PMP sampling captures the shape of the proxy whereas uniform sampling returns an overly optimistic set.  Fig. \ref{fig: learned_set} also depicts the prediction encoded into the safe set, which shrinks the feasible state space at high speeds. These results show that in order to navigate the turn, the vehicle must follow a racing line that maximizes curvature and clips the apex or preemptively reduce speed so that the outer track edge can be accessed in the turn. 

\begin{figure*}[t]
\vspace{0.2em}
    \centering
     \begin{subfigure}[t]{0.3\textwidth}
         \centering
         \includegraphics[scale = 0.25,trim={4cm 0 4cm 0}]{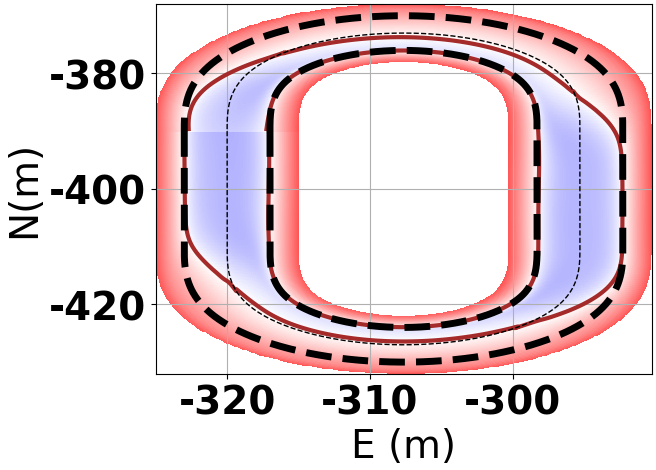}
         \caption{Model of Sec. \ref{sec: Net_Arch}.}
         \label{fig: base}
     \end{subfigure}
     \begin{subfigure}[t]{0.3\textwidth}
         \centering
         \includegraphics[scale = 0.25,trim={4cm 0 4cm 0}]{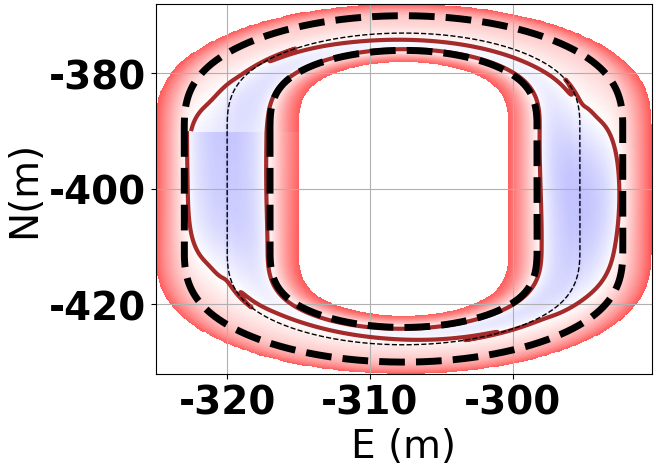}
         \caption{PMP model.}
         \label{fig: pmp}
     \end{subfigure}
    \begin{subfigure}[t]{0.3\textwidth}
         \centering
         \includegraphics[scale = 0.25,trim={4cm 0 4cm 0}]{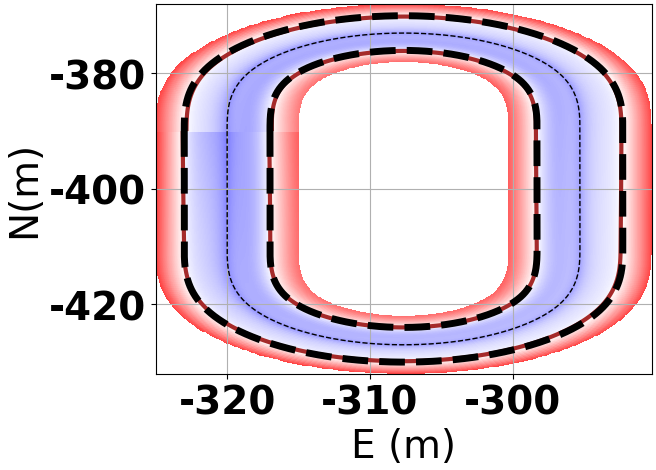}
         \caption{Uniform model.}
         \label{fig: uni}
     \end{subfigure}

    \caption{Value function at a speed of 10.5 m/s for $50k$ epochs and $30k$ samples for PMP sampling (center) and uniform sampling (right). Ground truth (Sec. \ref{sec: Net_Arch})  is shown on left. Brown solid line is the 0-level set of the value function and black dashed lines are track edges and centerline.}
    \label{fig: learned_set}
    \vspace{-1.5em}

\end{figure*}

\begin{table}[t]
\centering
\setlength{\tabcolsep}{3pt}
\begin{tabular}{llcc}
\toprule
Layers & Setting & PMP & Uniform \\
\midrule
2 & 30k epochs & $\mathbf{0.7 \pm 0.01}$ & $0.40 \pm 0.01$ \\
2 & 50k epochs & $\mathbf{0.73 \pm 0.01}$ & $0.46 \pm 0.04$ \\
2 & 70k epochs & $\mathbf{0.74 \pm 0.01}$ & $0.57 \pm 0.08$ \\
\midrule
2 & 45k samples & $\mathbf{0.72 \pm 0.01}$ & $0.52 \pm 0.02$ \\
2 & 60k samples & $\mathbf{0.73 \pm 0.01}$ & $0.51 \pm 0.04$ \\
\midrule
3 & 30k epochs & $\mathbf{0.6 \pm 0.03}$ & $0.51 \pm 0.05$ \\
3 & 50k epochs & $\mathbf{0.74 \pm 0.01}$ & $0.65 \pm 0.06$ \\
3 & 70k epochs & $\mathbf{0.75 \pm 0.005}$ & $0.72 \pm 0.03$ \\
\midrule
3 & 45k samples & $\mathbf{0.73 \pm 0.01}$ & $0.63 \pm 0.07$ \\
3 & 60k samples & $\mathbf{0.74 \pm 0.016}$ & $0.57 \pm 0.08$ \\
\bottomrule
\end{tabular}
\caption{IOU across model capacity and dataset size. Bold indicates where PMP sampling outperforms uniform sampling with statistical significance.}
\vspace{-2em}
\label{tab: IOU}
\end{table}

To quantitatively evaluate learned safe set representation, we measure the intersection-over-union (IOU) between the learned safe set and the ground truth proxy (Table \ref{tab: IOU}).
PMP sampling consistently yields higher IOU than uniform sampling, indicating improved approximation of the safe-set geometry. 
These results complement the results for failure rate and further demonstrate that PMP improves performance by concentrating samples on informative boundary trajectories, leading to more accurate safe set reconstruction. 

\vspace{-0.6em}
\subsection{Discussion}
From these results, we observe that PMP sampling improves performance in the most challenging tasks. Particularly PMP sampling provides the most benefit for settings with limited budget, data, and longer horizons. By focusing on safety-critical boundary trajectories, PMP sampling  improves the informativeness of the training distribution, leading to more efficient and reliable safety filter learning.

\begin{figure*}[t]
    \centering
     \begin{subfigure}[t]{0.32\textwidth}
         \centering
         \includegraphics[scale = 0.27,trim={5cm 0 5cm 0}]{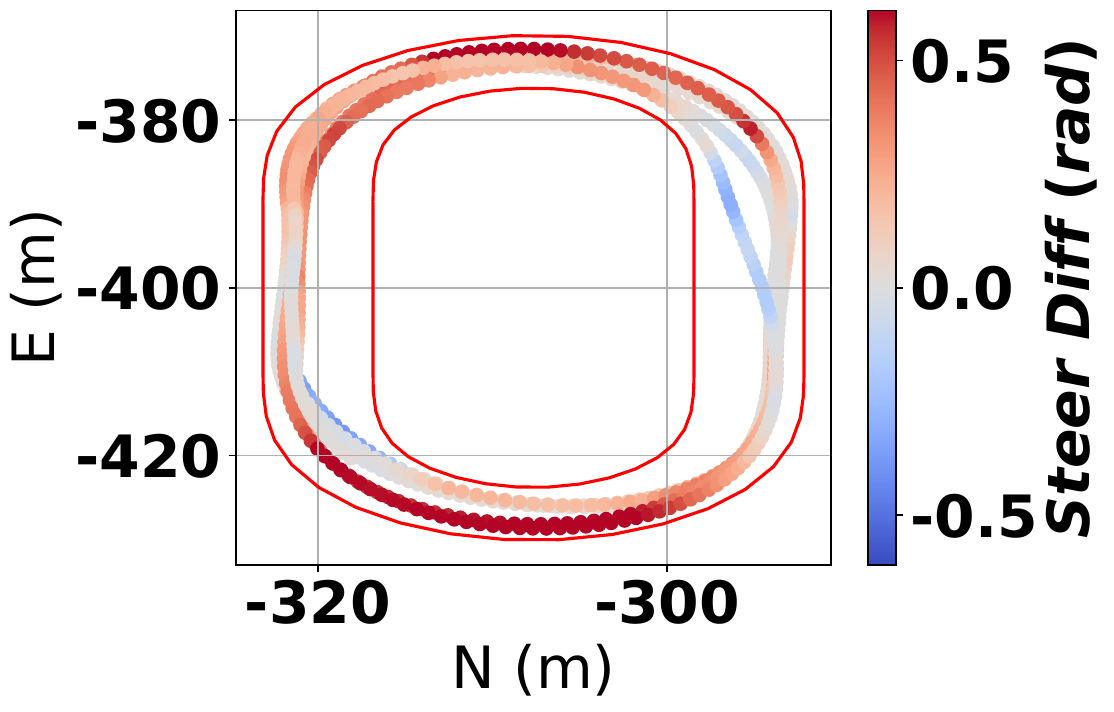}
         \caption{Steering difference.}
         \label{fig: delt_0pt1}
     \end{subfigure}
     \begin{subfigure}[t]{0.32\textwidth}
         \centering
         \includegraphics[scale = 0.27,trim={5cm 0 5cm 0}]{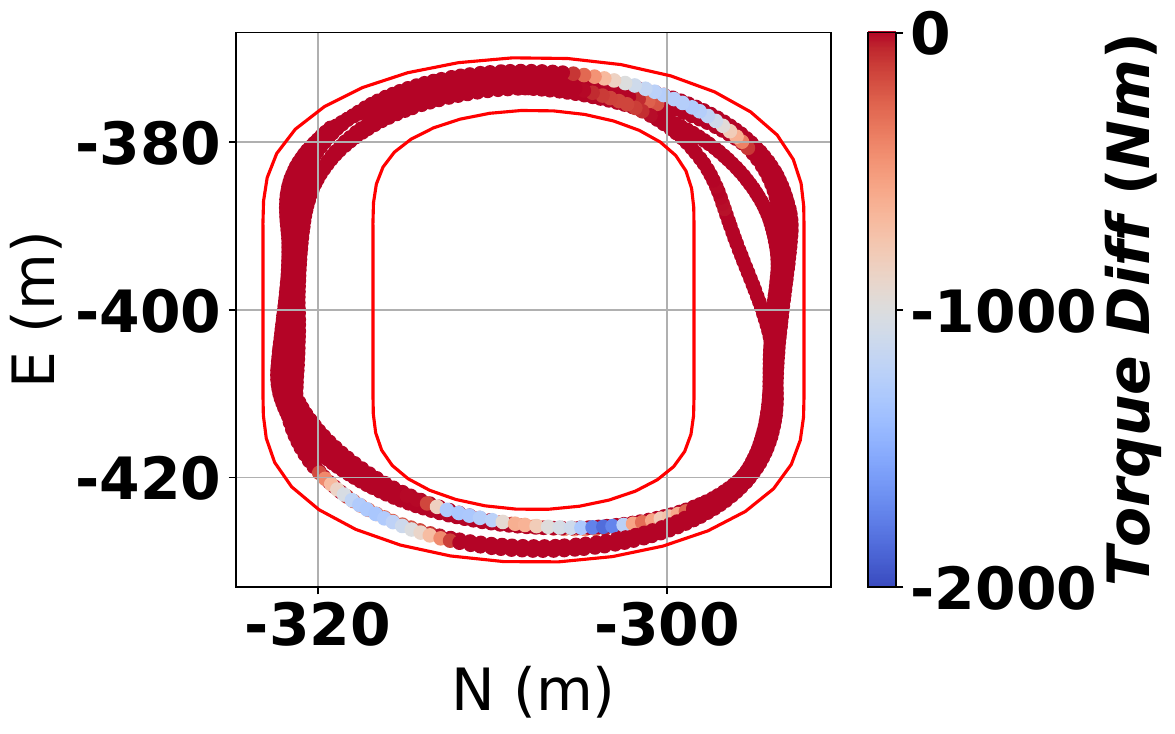}
         \caption{Axle torque difference.}
         \label{fig: tau_0pt1}
     \end{subfigure}
     \begin{subfigure}[t]{0.32\textwidth}
         \centering
         \includegraphics[scale = 0.27,trim={5cm 0 5cm 0}]{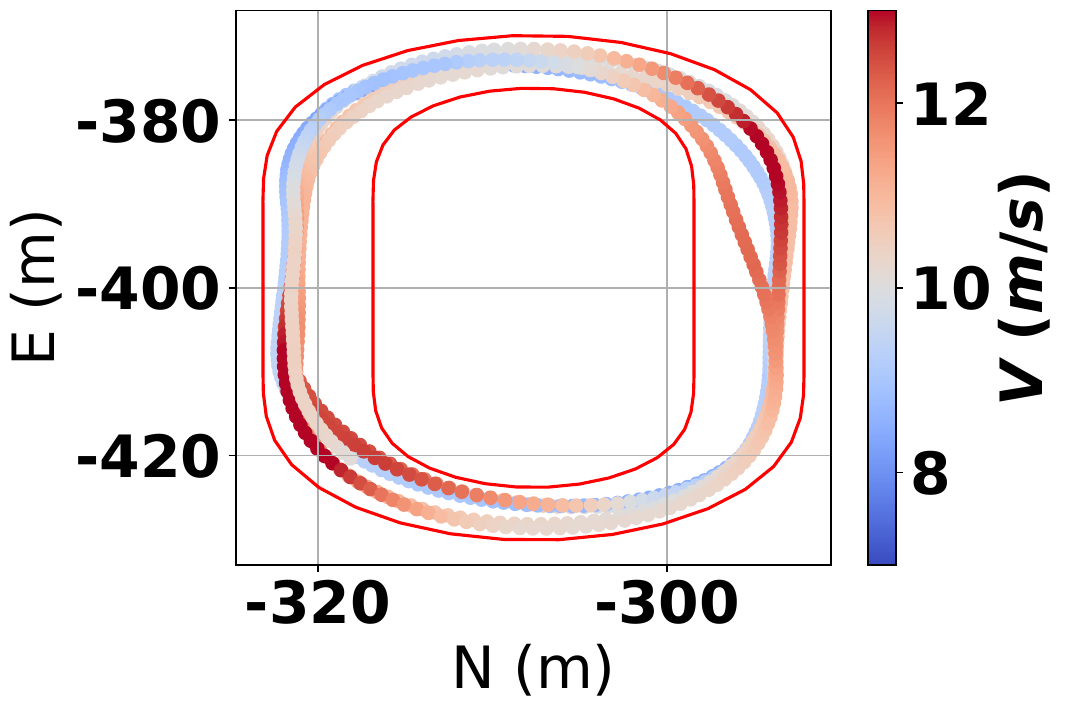}
         \caption{Velocity.}
         \label{fig: V_0pt1}
     \end{subfigure}
    \caption{Shared control experiment with $\gamma=0.1$ for steering difference (left) torque difference (center), and velocity (right). Color depicts intervention (left, center) and velocity magnitude (right).}
    \label{fig: exp_0pt1}
    \vspace{-1.5em}
\end{figure*}

\vspace{-0.6em}
\subsection{Experimental Validation} \label{sec: exp}

Next, we show that the approach yields predictive behavior and bridges the sim-to-real gap in shared control experiments where a driver attempts to race an oval track in the vehicle of Sec. \ref{sec: exp_veh} and the CBVF-QP filters commands to ensure safety. Fig. \ref{fig: delt_0pt1} shows the steering intervention as the difference between the CBVF-QP and driver steering request at the roadwheel, and Fig. \ref{fig: tau_0pt1} shows the torque intervention as the difference between the CBVF-QP and driver torque request. 

Throughout the experiment, the vehicle stays within the road edges (red solid lines), and proactive intervention occurs on the driver inputs. As the driver approaches the track edge, large steering interventions occur augmenting the driver steering request by up to $0.6 rad$ to prevent safety violations. Regarding throttle intervention, the CBVF-QP allows for acceleration on the straight (red) and rapidly reduces the commanded torque by over $1200 ~ Nm$ as the vehicle enters the turns which require a reduced speed to successfully navigate. This behavior leads to the speed reducing from $13.5 \,m/s$ on straights to $9.7 m/s$ as the turns are entered, as shown in Fig. \ref{fig: V_0pt1}. Importantly, these results demonstrate predictive behavior by slowing the vehicle preemptively to successfully navigate the turn, rather than operating in a myopic fashion. The minimal intervention is also observed, where the vehicle follows the driver intentions; e.g., on straights the driver steering and torque is respected and intervention primarily occurs at turns and where track edges are approached. 

Fig. \ref{fig: forward} shows a segment where the driver fails to successfully navigate a turn and the safety filter prevents track departure. The driver steering request is shown in green (\ref{fig: forward_delta}) and the output of the CBVF-QP is shown in blue. The measured vehicle position is shown in blue (\ref{fig: forward_e}).  The CBVF-QP prevents the vehicle from violating safety and leaving the track edge (red). In contrast, a post-hoc simulation that utilizes the vehicle model of Sec. \ref{sec:app} to simulate the response of the vehicle had the driver inputs been directly applied leads to the vehicle departing outside the outer track edge (green) due to inadequate steering while navigating the turn. These results demonstrate the utility of the proposed approach to promote safety in real world experiments. Finally, wall times are on the order of $3 ms$ yielding real-time applicability.

\begin{figure}[]
    \centering
     \begin{subfigure}[t]{0.49\columnwidth}
         \centering
         \includegraphics[scale = 0.24,trim={5cm 0 5cm 0}]{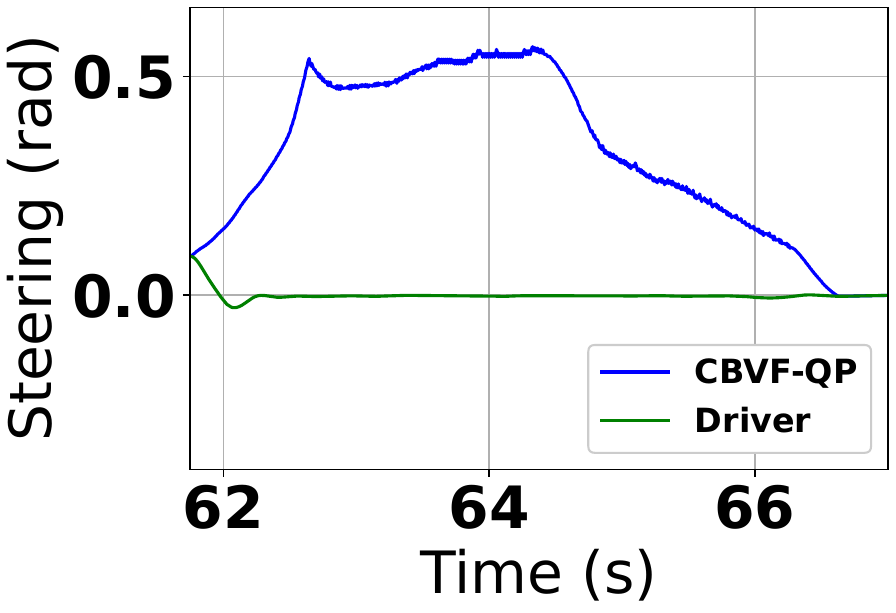}
         \caption{Steering.}
         \label{fig: forward_delta}
     \end{subfigure}
     \begin{subfigure}[t]{0.49\columnwidth}
         \centering
         \includegraphics[scale = 0.24,trim={5cm 0 5cm 0}]{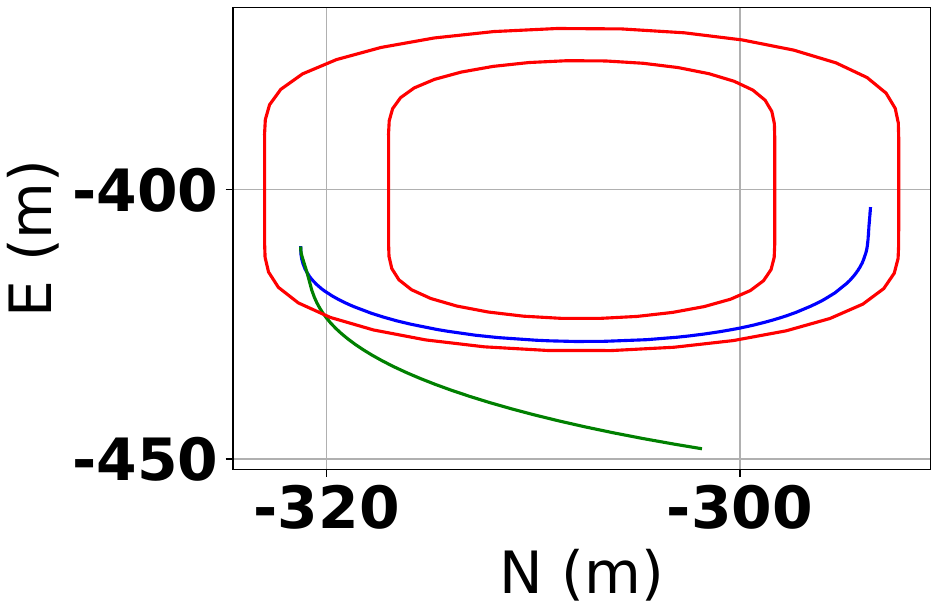}
         \caption{Position.}
         \label{fig: forward_e}
     \end{subfigure}
    \caption{Segment of $\gamma=0.1$ experiment for steering (left) and position (right). CBVF-QP output is shown in blue and driver request in green). Track edges are shown in red.}
    \label{fig: forward}
    \vspace{-1.5em}
\end{figure}

\vspace{-1.5em}
\section{Conclusion} \label{sec:conc}
This paper presents an approach to reduce the myopic behavior of safety filters and improve training efficiency and accuracy of learned HJ reachable sets. 
Conditions based on Pontryagin's Maximum Principle enable generation of boundary defining trajectories, augmenting the training set to improve training efficiency.
These trajectories enable learning a CBVF that directly encodes future trajectories into the safe set in a efficient manner despite high dimensionality and complex dynamics.
A minimally invasive quadratic program for shared control of vehicles in a racing scenario is presented based upon the learned CBVF.
Simulations and experiments demonstrate the approach yields proactive control augmentation and improves learning efficiency, yielding faster convergence, reduced failure rates, and improved safe set reconstruction.

In future work,  it would be interesting to further investigate the relationship between sampling schemes for learning, the geometry of safe sets, and controllability properties of dynamical systems. For instance, uniform sampling around $\partial{\mathcal{C}}$ might work well for the system $\dot{x}_t=u_t$, whereas systems with complicated dynamics and limited control authority like in \eqref{eq:bike} require more sophisticated schemes.

\appendix \label{appendix}

We present a method to account for box-constrained control inputs for dynamically extended systems, by deriving an equivalent representation of the dynamics without explicitly enforcing constraints on the states corresponding to the dynamically extended input $u$.

A dynamically extended system with state $(x,u)$ and control input $\dot{u}$ takes the control-affine form
\begin{align*}
\begin{bmatrix}
\dot x\\[-1mm] \dot u
\end{bmatrix}&=f\left(\begin{bmatrix}
x\\[-1mm]u
\end{bmatrix}\right)+g\left(\begin{bmatrix}
x\\[-1mm]u
\end{bmatrix}\right)\dot u, \\
    u&\in[-\overline{u},\overline{u}], ~
    \quad\\
    \dot u
    &\in\mathbb{U}=[-\overline{\dot u},\overline{\dot u}].
\end{align*}

\textit{New system}: 
Applying the PMP directly to this system would require considering the state constraints $u\in[-\overline{u},\overline{u}]$. Instead, we propose an equivalent representation of this system that implicitly enforces the original state constraints on $u$.
%
%
%
We express the modified system as:  
$$
\begin{bmatrix}
\dot x\\[-1mm] \dot u
\end{bmatrix}=f(x, u)+g(x, u)\dot u=\begin{bmatrix}
f_1(x)\\\vdots \\f_n(x)\\f_u(x, u)
\end{bmatrix}
+
\begin{bmatrix}
0\\\vdots\\ 0\\g_u(x, u)
\end{bmatrix}\dot u
$$
\noindent
where $f_i(x) ~ \forall i \in\{1, \dots, n\}$ denote the dynamics of the original state variables, and the final state corresponds to the dynamically extended input $u$, with dynamics defined as:
\begin{align*}
f_u(x)&=\begin{cases}
    0
    &\text{if } u \in(-\overline{u},\overline{u})
    \\
    \frac{\overline{\dot u}}{2}
    &\text{if }u=-\overline{u}
    \\
    -\frac{\overline{\dot u}}{2}
    &\text{if }u=\overline{u},
    \end{cases}
    \\
    g_u(x)&=\begin{cases}
    1
    &\text{if }u\in(-\overline{u},\overline{u})
    \\
    \frac{1}{2}
    &\text{if }u=-\overline{u}\text{ or }u=\overline{u},
    \end{cases}
\end{align*}
The dynamics of $u$ are now
$$
\dot u_t = \begin{cases}
    \dot u
    &\text{if }u\in(-\overline{u},\overline{u})
    \\
    \frac{\overline{\dot u}+\dot u}{2}
    &\text{if }u=-\overline{u}
    \\
    \frac{-\overline{\dot u}+\dot u}{2}
    &\text{if }u=\overline{u}.
    \end{cases}
\text{with }\dot u\in[-\overline{\dot u},\overline{\dot u}]
$$
This representation implicitly enforces the box constraints on $u$ within the dynamics. We have that:
\begin{itemize}
    \item If $u \in(-\overline{u},\overline{u})$, then no constraint is active.
    \item If $u=-\overline{u}$, then $\dot u_t\in[0,\overline{\dot u}]\geq 0$.
    \item If $u=\overline{u}$, then $\dot u_t\in[-\overline{\dot u},0]\leq 0$.
\end{itemize}
Furthermore, a smooth approximation (e.g., tanh) can be applied to satisfy the smoothness assumptions of Sec. \ref{sec:approach}. This implicit representation allows for state constraints to be easily accounted for and to use the standard PMP that underlies our proposed data augmentation scheme.

\bibliographystyle{IEEEtran}

\bibliography{References}

\begin{thebibliography}{10}
\providecommand{\url}[1]{#1}
\csname url@samestyle\endcsname
\providecommand{\newblock}{\relax}
\providecommand{\bibinfo}[2]{#2}
\providecommand{\BIBentrySTDinterwordspacing}{\spaceskip=0pt\relax}
\providecommand{\BIBentryALTinterwordstretchfactor}{4}
\providecommand{\BIBentryALTinterwordspacing}{\spaceskip=\fontdimen2\font plus
\BIBentryALTinterwordstretchfactor\fontdimen3\font minus \fontdimen4\font\relax}
\providecommand{\BIBforeignlanguage}[2]{{%
\expandafter\ifx\csname l@#1\endcsname\relax
\typeout{** WARNING: IEEEtran.bst: No hyphenation pattern has been}%
\typeout{** loaded for the language `#1'. Using the pattern for}%
\typeout{** the default language instead.}%
\else
\language=\csname l@#1\endcsname
\fi
#2}}
\providecommand{\BIBdecl}{\relax}
\BIBdecl

\bibitem{WABERSICH2021109597}
K.~P. Wabersich and M.~N. Zeilinger, ``A predictive safety filter for learning-based control of constrained nonlinear dynamical systems,'' \emph{Automatica}, vol. 129, p. 109597, 2021.

\bibitem{ames2019control}
A.~D. Ames, S.~Coogan, M.~Egerstedt, G.~Notomista, K.~Sreenath, and P.~Tabuada, ``Control barrier functions: Theory and applications,'' in \emph{European Control Conference}, 2019, pp. 3420--3431.

\bibitem{AVEC2024}
C.~Jiang, H.~Gan, I.~V{\"o}r{\"o}s, D.~Tak{\'a}cs, and G.~Orosz, ``Safety filter for lane-keeping control,'' in \emph{16th International Symposium on Advanced Vehicle Control}, 2024, pp. 1--6.

\bibitem{Dallas_2025}
J.~Dallas, J.~Talbot, M.~Suminaka, M.~Thompson, T.~Lew, G.~Orosz, and J.~Subosits, ``Control barrier functions for shared control and vehicle safety,'' in \emph{American Control Conference}.\hskip 1em plus 0.5em minus 0.4em\relax IEEE, Jul. 2025, p. 4203–4210.

\bibitem{GARG2024100945}
K.~Garg, J.~Usevitch, J.~Breeden, M.~Black, D.~Agrawal, H.~Parwana, and D.~Panagou, ``Advances in the theory of control barrier functions: Addressing practical challenges in safe control synthesis for autonomous and robotic systems,'' \emph{Annual Reviews in Control}, vol.~57, p. 100945, 2024.

\bibitem{dtcbf}
J.~Zeng, B.~Zhang, and K.~Sreenath, ``Safety-critical model predictive control with discrete-time control barrier function,'' in \emph{American Control Conference}, 2021, pp. 3882--3889.

\bibitem{TSCBF}
M.~Vahs, R.~I.~C. Muchacho, F.~T. Pokorny, and J.~Tumova, ``Forward invariance in trajectory spaces for safety-critical control,'' in \emph{IEEE International Conference on Robotics and Automation}, 2025, pp. 3926--3932.

\bibitem{HJOverview}
S.~Bansal, M.~Chen, S.~Herbert, and C.~J. Tomlin, ``Hamilton-jacobi reachability: A brief overview and recent advances,'' in \emph{IEEE Conference on Decision and Control}, 2017.

\bibitem{CBVF}
J.~J. Choi, D.~Lee, K.~Sreenath, C.~J. Tomlin, and S.~L. Herbert, ``Robust control barrier–value functions for safety-critical control,'' in \emph{IEEE Conference on Decision and Control}, 2021.

\bibitem{deepreach}
S.~Bansal and C.~J. Tomlin, ``Deepreach: A deep learning approach to high-dimensional reachability,'' in \emph{IEEE International Conference on Robotics and Automation}, 2021, pp. 1817--1824.

\bibitem{RAISSI}
M.~Raissi, P.~Perdikaris, and G.~Karniadakis, ``Physics-informed neural networks: A deep learning framework for solving forward and inverse problems involving nonlinear partial differential equations,'' \emph{Journal of Computational Physics}, vol. 378, pp. 686--707, 2019.

\bibitem{MCCLENNY2023111722}
L.~D. McClenny and U.~M. Braga-Neto, ``Self-adaptive physics-informed neural networks,'' \emph{Journal of Computational Physics}, vol. 474, p. 111722, 2023.

\bibitem{feng2025}
Z.~Feng, L.~Qiu, and S.~Bansal, ``Bridging model predictive control and deep learning for scalable reachability analysis,'' in \emph{Robotics: Science and Systems}, 2025.

\bibitem{oh2025safetyagencyhumancenteredsafety}
D.~D. Oh, J.~Lidard, H.~Hu, H.~Sinhmar, E.~Lazarski, D.~Gopinath, E.~S. Sumner, J.~A. DeCastro, G.~Rosman, N.~E. Leonard, and J.~F. Fisac, ``Safety with agency: Human-centered safety filter with application to ai-assisted motorsports,'' 2025.

\bibitem{Agrachev2004}
A.~A. Agrachev and Y.~L. Sachkov, \emph{Control Theory from the Geometric Viewpoint}.\hskip 1em plus 0.5em minus 0.4em\relax Springer Berlin Heidelberg, 2004.

\bibitem{Trelat2012}
E.~Tr{\'{e}}lat, ``Optimal control and applications to aerospace: Some results and challenges,'' \emph{Journal of Optimization Theory \& Applications}, vol. 154, no.~3, pp. 713--758, 2012.

\bibitem{BonnardChyba2003}
B.~Bonnard and M.~Chyba, \emph{Singular Trajectories and their Role in Control Theory}.\hskip 1em plus 0.5em minus 0.4em\relax Springer Berlin Heidelberg, 2003.

\bibitem{Kurzhanski2000}
A.~B. Kurzhanski and P.~Varaiya, ``Ellipsoidal techniques for reachability analysis,'' in \emph{Hybrid Systems: Computation and Control}, 2000.

\bibitem{Natherson2025}
R.~A. Natherson and D.~J. Scheeres, ``Study of methods for computing directional reachability,'' \emph{{Journal of Guidance, Control, and Dynamics}}, vol.~48, no.~12, pp. 2860--2869, 2025.

\bibitem{lew2024convexhullsreachablesets}
T.~Lew, R.~Bonalli, and M.~Pavone, ``Convex hulls of reachable sets,'' \emph{IEEE Transactions on Automatic Control}, vol.~70, no.~12, pp. 8195--8209, 2025.

\bibitem{kim2025reachabilitybarriernetworkslearning}
M.~Kim, W.~Sharpless, H.~J. Jeong, S.~Tonkens, S.~Bansal, and S.~Herbert, ``Reachability barrier networks: Learning hamilton-jacobi solutions for smooth and flexible control barrier functions,'' 2025.

\bibitem{LewEstimatingCH2024}
T.~Lew, R.~Bonalli, L.~Janson, and M.~Pavone, ``Estimating the convex hull of the image of a set with smooth boundary: Error bounds and applications,'' \emph{Discrete \& Computational Geometry}, vol.~74, no.~1, pp. 203--241, 2024.

\bibitem{Lee2012}
J.~M. Lee, \emph{Introduction to Smooth Manifolds}, 2nd~ed.\hskip 1em plus 0.5em minus 0.4em\relax Springer New York, 2012.

\bibitem{Wang2026}
Y.~Wang, C.~Hu, and Z.~Jin, ``Time-optimal switching surfaces for triple integrator under full box constraints,'' in \emph{American Control Conference}, 2026.

\bibitem{Boscain2005}
U.~Boscain and B.~Piccoli, ``An introduction to optimal control,'' \emph{Contrôle non linéaire et applications}, pp. 19--66, 2005.

\bibitem{dallas2023hierarchical}
M.~Thompson, J.~Dallas, J.~Y.~M. Goh, and A.~Balachandran, ``Adaptive nonlinear model predictive control: Maximizing tire force and obstacle avoidance in autonomous vehicles,'' \emph{IEEE Transactions on Field Robotics}, vol.~1, pp. 318--331, 2024.

\bibitem{Pacejka_2005}
H.~Pacejka, \emph{Tire and Vehicle Dynamics}.\hskip 1em plus 0.5em minus 0.4em\relax Elsevier, 2005.

\end{thebibliography}

\end{document}